\documentclass{article}

\usepackage[preprint]{neurips_2025}

\usepackage[utf8]{inputenc} 
\usepackage[T1]{fontenc}    
\usepackage{hyperref}       
\usepackage{url}            
\usepackage{booktabs}       
\usepackage{amsfonts}       
\usepackage{nicefrac}       
\usepackage{microtype}      
\usepackage{xcolor}         
\usepackage{algorithm} 
\usepackage{algorithmic}  
\usepackage[algo2e]{algorithm2e} 
 \usepackage{amsmath}
\usepackage{amssymb}
\usepackage{amsthm}
\usepackage{framed}
\usepackage{xcolor}
\usepackage{tcolorbox}
\usepackage{adjustbox}
\newcommand{\cmark}{\textcolor{green!60!black}{\ding{51}}} 
\newcommand{\xmark}{\textcolor{red}{\ding{55}}}           
\usepackage{pifont} 
\usepackage{natbib}
\tcbuselibrary{most}

\newtheorem{theorem}{Theorem}
\newtheorem{proposition}[theorem]{Proposition}
\newtheorem{lemma}{Lemma}

\newtheorem{remark}{Remark}
\newtcolorbox{mybox}[2][]{colbacktitle=red!10!white, colback=blue!10!white,coltitle=red!70!black, title={#2},fonttitle=\bfseries,#1}

\title{\textsc{CoVeR}: Conformal  Calibration for Versatile and Reliable Autoregressive Next-Token Prediction}

\author{%
  Yuzhu Chen\textsuperscript{1}\thanks{Both authors contributed equally.}
  \quad
  Yingjie Wang\textsuperscript{2}\footnotemark[1]
  \quad
  Shunyu Liu\textsuperscript{2}
  \quad
  Yongcheng Jing\textsuperscript{2}
  \quad
  Dacheng Tao\textsuperscript{2}\\
  $^1$University of Science and Technology of China\quad
  $^2$Nanyang Technological University
 }

\begin{document}

\maketitle

\begin{abstract}
Autoregressive pre-trained models combined with decoding methods have achieved impressive performance on complex reasoning tasks. While mainstream decoding strategies such as beam search can generate plausible candidate sets,  they often lack provable coverage guarantees, and struggle to effectively balance search efficiency with the need for versatile trajectories, particularly those involving long-tail sequences that are essential in certain real-world applications. To address these limitations, we propose \textsc{CoVeR}, a novel model-free decoding strategy wihtin the conformal prediction framework that simultaneously maintains a compact search space and ensures high coverage probability over desirable  trajectories. Theoretically, we establish a PAC-style generalization bound, guaranteeing that \textsc{CoVeR} asymptotically achieves a coverage rate of at least $1 - \alpha$ for any target level $\alpha \in (0,1)$. 
\end{abstract}

\section{Introduction}
Autoregressive pre-trained  models have emerged as state-of-the-art tools for natural language generation. In particular, the development of autoregressive pre-trained models for reasoning such as OpenAI’s o-[$n$] series \citep{jaech2024openai,o3_openai} and DeepSeek’s R1 \citep{guo2025deepseek} has sparked a growing body of research on Chain-of-Thought (CoT) reasoning. These models have demonstrated significant improvements in mathematical reasoning, programming tasks, and commonsense inference \citep{sun2023survey, Wei-NeurIPS-models-2022}. Remarkably, some have even been reported to exhibit behavior consistent with passing the Turing Test  \citep{jones2025large}. During generation, the decoding strategy plays a critical role in shaping the quality of sequence generation in autoregressive pre-trained models. Common approaches include greedy decoding, beam search \citep{xie2023self, zhu2024deductive}, and various enhanced variants \citep{sridhar2022improved, xie2023self, zhao2022calibrating, chan2025efficient, yang2024language}. Despite their impressive empirical success, these methods often lack provable coverage guarantees and fail to flexibly capture long-tail reasoning trajectories that are critical in specific real-world scenarios. 

In statistical learning community,  \emph{Conformal Prediction} (CP) offers a promising framework for achieving  provable coverage guarantee. By using an external calibration set, CP  produce prediction sets whose size adapts to the data while guaranteeing finite-sample, model-agnostic, and distribution-free coverage  \citep{shafer2008tutorial, balasubramanian2014conformal, tibshirani2019conformal, ding2023class}.  This characteristic makes CP particularly appealing for generative autoregressive models, where verifying the correctness of an output from a set is often easier than generating the correct output directly \citep{huang-etal-2023-large,wang2022selfa}.

\begin{table}[htp!]
\centering
\begin{adjustbox}{max width=\textwidth}
\begin{tabular}{lcccccc}
\hline
Method & Long-tail Seq. & Global Cov. & Local Cov.  & Compact Search  & PAC Bound \\
\hline
\cite{ravfogel2023conformal}    & \xmark & \cmark & \cmark  & \xmark  & \xmark \\
 \cite{ren2023robots}         & \xmark & \cmark & \xmark & \cmark & \xmark  \\
\cite{quach2023conformal}       & \xmark & \cmark & \xmark & \cmark & \cmark\\
 \cite{deutschmann2024conformal} & \xmark & \cmark & \cmark & \xmark & \xmark\\
\cite{tang2023less} & \cmark & \xmark & \xmark & \xmark & \xmark\\
 \cite{li2023search} & \cmark & \xmark & \xmark & \cmark & \xmark\\
Our Method      & \cmark & \cmark & \cmark & \cmark & \cmark  \\
\hline
\end{tabular}
\end{adjustbox}
\caption{Comparison of decoding and conformal prediction methods across different properties. }
\label{summarization}
\end{table}

Recent advances in applying conformal prediction to sequence generation have primarily focused on sampling-based approximations \citep{quach2023conformal} and human-assisted pruning strategies \citep{ren2023robots}, both providing sequence-level coverage guarantees. Notably, \cite{deutschmann2024conformal} introduced an extension of conformal prediction to token-level autoregressive generation.  Although these methods achieve both empirical and theoretical success in conformal sequence generation, they exhibit a tendency to prioritize mainstream outputs while neglecting long-tail sequences that may be essential in  specific real-world scenarios.


To address this gap, we propose \textsc{CoVeR}, a \textbf{Co}nformal  Calibration for \textbf{Ve}rsatile and \textbf{R}eliable autoregressive next-token prediction.  Our main contributions are summarized as follows:
\begin{itemize}
    \item \emph{Algorithmic Design.}  
 We propose \textsc{CoVeR}, a novel model-free decoding strategy that integrates conformal prediction into a dual-objective optimization framework.  This framework simultaneously maintain a compact search space and enables pretrained models to selectively generate either high-frequency or long-tail sequences with provable accuracy.
    \item \emph{Theoretical Guarantees.} We derive a PAC-style generalization bound for \textsc{CoVeR}, providing finite-sample coverage guarantees at both the global (full-sequence) and local (step-wise) levels of sequence generation.
\end{itemize}

\section{Related Work}
\paragraph{Conformal Prediction for Sequence Generation.}
The complexity of sequence generation makes applying CP to autoregressive models nontrivial. Early work addressed constrained tasks such as single-token prediction \citep{ravfogel2023conformal, ding2023class, tibshirani2019conformal}, while more recent efforts have extended CP to full sequence generation \citep{ren2023robots, quach2023conformal, deutschmann2024conformal}. From a generative modeling perspective, \cite{ren2023robots} constructs a product of token-level sets for task planning, suitable for deterministic outputs, while \cite{quach2023conformal} uses sampling-based decoding with semantic filtering via Learn-then-Test \citep{angelopoulos2021learn} to support sequence-level accept/reject decisions. However, their approach contrasts with our token-by-token calibration. Closest to our setting, \cite{deutschmann2024conformal} studies dynamic conformal sets for each token generation, but our method differs in three key aspects: (i) it flexibly supports both frequent and long-tail  paths; (ii) it introduces dual-objective optimization objective to jointly optimize global and local noncoverage rate, avoiding the exponential coverage decay $(1-\alpha)^L$ in \citep{deutschmann2024conformal}; and (iii) it provides data-dependent PAC-style generalization bounds, offering tighter  guarantees.

\paragraph{Decoding Strategy for Pre-trained Model Reasoning.}
In autoregressive pre-trained models, the decoding strategy plays an important role in determining sequence quality. Greedy decoding, though simple and efficient, often fails to recover from early suboptimal choices, leading to incoherent sequences \citep{chiang2021relating}. Beam search mitigates this issue by maintaining multiple partial hypotheses based on cumulative log probabilities, producing more accurate and coherent outputs \citep{xie2023self,zhu2024deductive}. Recent extensions to beam search aim to improve generation reliability \citep{sridhar2022improved,xie2023self,zhao2022calibrating,chan2025efficient,yang2024language}, such as stepwise self-evaluation \citep{xie2023self} and sequence likelihood calibration \citep{zhao2022calibrating}.  Nevertheless, these approaches still offer only limited flexibility in exploring diverse sequences \cite{chen2023say}, motivating an increasing work on more versatile decoding strategies \citep{minsky1997negative, chen2023say, li2023search}.  \cite{tang2023less}  introduces a contrastive learning scheme to elicit outputs that are plausible yet less likely, while \cite{li2023search} proposes a variable-wise prompting framework to systematically generate factually correct but long-tail inferential statements. In parallel, other studies focus on long-tail data generation for probing large models:  \cite{zhou2020rica} creates self-contained commonsense statements featuring novel entities, and  \cite{arnaout2022uncommonsense} evaluates models on rare negative knowledge.  However, these approaches often require careful rollout design and lack provable theoretical coverage guarantees like comformal prediction, limiting their scalability and reliability in real-world applications.

A comprehensive comparison of our method against representative decoding and conformal prediction approaches is presented in Table~\ref{summarization}, highlighting key differences across different dimensions. Here, Long-tail Seq. indicates support for long-tail sequence generation; Global Cov. denotes sequence-level global coverage guarantees; Local Cov. reflects token-level (stepwise) coverage guarantees; Compact Search indicates maintenance of a compact decoding search space via some regulation techniques; PAC Bound denotes the presence of a PAC-style theoretical guarantee.

\section{Preliminary}
Suppose each input variable $X \in \mathcal{X}$ is associated with a class label $Y \in \mathcal{Y}$, where $\mathcal X$ and $\mathcal{Y}$ are input and output space, respectively. Let $D=\{(x_i, y_i)\}_{i=1}^N$ be a calibration dataset consisting of $N$ samples drawn from an unknown distribution. Given a new test point $(X, Y)$ drawn from the same distribution, the goal of conformal prediction is to construct a prediction set $\mathcal{C}(X)$ such that the coverage condition ${\rm Pr}(Y\in \mathcal{C}(X))\geq 1-\alpha$ is satisfied for a user-specified confidence level $\alpha \in [0, 1]$. Let $\sigma : \mathcal{X} \times \mathcal{Y} \to \mathbb{R}$ denote a conformal score function. The score function is typically derived from a pre-trained classifier $f$,  e.g., $\sigma(x, y) = f_y(x)$, where $f_y(x)$ is the $y$-th softmax output of $f$ on $x$.  For brevity, let $r_i = \sigma(x_i, y_i)$ be the score of the $i$-th calibration data point. For $\tau \in [0,1]$ and a finite set $A \subseteq \mathbb{R}$, let $\mathrm{Quantile}(\tau, A)$ denote the smallest $a \in A$ such that the fraction of elements in $A$ that are bigger than or equal to $a$ is at least $\tau$. If no such $a$ exists, set $\mathrm{Quantile}(\tau, A) = +\infty$.

\paragraph{Conformal Prediction} The standard conformal prediction sets are given by:
\begin{equation}\label{classic_conformal}
\mathcal{C}(X) = \{ y \in \mathcal{Y}: \sigma(X, y) \geq \hat{q} \}~~~{\rm where}~~~    \hat{q} = \mathrm{Quantile}\left( \frac{(N+1) \alpha}{N}, \{r_i\}_{i=1}^N \right).
\end{equation}

The conformal prediction set provides finite-sample coverage guarantees, i.e., Pr$(Y\in\mathcal{C}(X))\geq1-\alpha, \forall  \alpha\in (0,1)$. Recently, it has been successfully extended to autoregressive inference methods, such as Beam Search \citep{deutschmann2024conformal}.






\paragraph{Conformal Prediction for Beam Search.} Given a pair $(X, \mathbf S)$, where $\mathbf S$ is a structured output sequence instead of one-dimensional label, we wish to produce a conformal set $\mathcal{C}(X)$ of candidate sequences such that ${\rm Pr}(\mathbf S \in \mathcal{C}(X))\geq 1-\alpha$. A straightforward strategy begins by performing  beam search to obtain a reasonably-sized set of proposals $\beta(X)$, and to then predict $\mathcal{C}(X) \subset \beta(X)$.  Define the in-beam subgroup of the calibration data  by  
$$
\mathcal{B} = \{(x_i, \mathbf{s}_i) \in D \mid \mathbf{s}_i \in \beta(x_i)\}.
$$
By performing split-CP calibration on this subgroup, a threshold $\hat q$ is obtained by calculating the quantile as shown in Eq. \eqref{classic_conformal}. At inference time, given a test sample $(X, \mathbf{S})$, the prediction set is defined by 
\begin{equation}
    \mathcal{C}(X) = \{\mathbf S' \in \beta(X) \mid \sigma(X, \mathbf S') \geq \hat q\}.
\end{equation}
On this prediction set, the coverage property is guaranteed by  following:
\begin{lemma}
\textit{With probability at least \( 1 - \delta,~ 0<\delta<1,\)} there holds
\[
\mathbb{P}\left(\mathbf S \in \mathcal{C}(X)\right) \geq (1 - \alpha)\, B\left(\delta; |\mathcal{B}|, N + 1 - N_\beta\right),
\]
\textit{where \( B(a, b) \) is a beta distribution and \( B(\delta; a, b) \) is its \( \delta \)-quantile.}
\end{lemma}

However, this approach only works if one can ensure a valid coverage guarantee over the set $\beta(X)$ produced by beam search. Moreover, it functions merely as a post-hoc verifier for beam search and does not guide the construction of prediction sets at each step. In addition, \cite{deutschmann2024conformal} propose a dynamic conformal set for beam search, which adaptively adjusts the beam size at each decoding step based on a conformal prediction threshold.

\textbf{Dynamic Conformal Beam Search.} Consider a family of conformal scores  that can be evaluated on length-$l$ sequences, denoted by $r_i^{1:l}=\sigma(x_i, \mathbf{s}_i^{1:l})$, where $\mathbf{s}_i^{1:l}$ is the length-$l$ truncation of $\mathbf s_i$. Selecting the initial sample set as $\mathcal{B}^{(0)}$ and denoting by $N_0=|\mathcal{B}^{(0)}|$, this method specifies a per-step confidence level $1 - \alpha$ and calibrate iteratively as follows: at the $l$-th step,
\begin{enumerate}
    \item Define $k_\alpha^l = \lfloor (N_{l-1} + 1) \alpha \rfloor$.
    \item Order the calibration set by increasing length-$l$ scores $r^{1:l}_1 \le \dots \le r^{1:l}_{N_{l-1}}$, and define $    \hat{Q}_{l} = r^{1:l}_{k_\alpha^{l}}$
    \item Set $N_l = |\mathcal{B}^{(l-1)}| - k_\alpha^{(l)}$, $\mathcal{B}^{(l)} = \{(x_i, \mathbf s_i)\}_{i < k_\alpha^{(l)}}$.
\end{enumerate}

At inference, proceeding iteratively until all sequences in $\mathcal{C}_\alpha^{(l)}(X)$ terminate or a maximum length $L$ is reached,  the next conformal beam as all continuations in the previous beam that pass the next threshold is defined by:
\begin{equation}
\mathcal{C}^{(l)}(X) = \left\{
\mathbf S^{1:l-1} a \ \middle| \
\begin{array}{l}
a \in A,\ \mathbf S^{1:l-1} \in \mathcal{C}^{(l-1)}(X), \sigma(X, \mathbf S^{1:l-1} a) \ge \hat{Q}_{l}
\end{array}
\right\},
\end{equation}
where $S^{1:0}a:=a$ for the first step $l=1$.
This approach mirrors the traditional beam-search algorithm in that it keeps a set of proposals at each  step and generates a set of high-scoring continuations of the current proposal for the next step with the following theoretical guarantee:  

\begin{lemma}\label{dynamic_beam}
With the maximum length \( L\) of decoding step and exchangeable data \( (X, \mathbf S)\), there holds $\mathbb{P} \left[ \mathbf S \in \mathcal C^{(L)}(X) \right] \geq (1 - \alpha)^{L}$.
\end{lemma}

Despite offering conformal validity for each step, existing conformal  methods for multi-step next-token generation still exhibit certain limitations in terms of feasibility and practical reliability: a) the overall performance is highly dependent on the quality of the beam search itself. For example, long-tail solutions may be prematurely discarded in early steps due to suboptimal intermediate scores; b) the cumulative coverage guarantee, given by $(1-\alpha)^L$ (cf. Lemma~\ref{dynamic_beam}), becomes increasingly weak as the maximum reasoning length  $L$ grows.

\section{CoVeR: Conformal Validity for Multi-step Reasoning}
In this section, Section \ref{sect_coveragedecomposition} decomposes the global sequence-level non-coverage rate into a weighted sum of local non-coverage terms, motivating the core idea that flexible generation can be achieved by controlling these local guarantees. Sections \ref{sect_clustering}–\ref{sect_optimizationobjective} describe the core methodologies, including the token-based clustering procedure that defines local regions and introduce the resulting optimization objective. Section \ref{sect_optimization} then develops an efficient algorithm for solving this objective. Finally, Section \ref{sect_pacbound} provides a PAC-style bound, demonstrating finite-sample guarantees and asymptotic convergence of the proposed approach.

\subsection{Non-Coverage Rate Decomposition}\label{sect_coveragedecomposition}

Let the input be \( X \in \mathcal{X} \), and let the maximum-length-$L$ reasoning sequence be $\mathbf{S}:= \mathbf S^{1:L}  = (S^1 \cdots S^L) \in \mathcal{S}$, where sequence is terminated by a disjoint token, and each token \( S^l \) lies in a \( |\mathcal{Y}| \)-dimensional classification space.  Our main goal is to construct a conformal set \( \mathcal{C}(X) \) such that it satisfies a user-specified coverage criterion, either $\mathbb{P}[\mathbf{S} \notin \mathcal{C}(X)] < \alpha$ or equivalently $\mathbb{P}[\mathbf{S} \in \mathcal{C}(X)] > 1 - \alpha$ for any  \( \alpha \in (0,1) \).

Constructing the conformal set only from a global coverage perspective limits control over generation diversity, as a single threshold favors high-probability patterns and overlooks step-level semantics. Local coverage, however, can be defined in various ways. In our approach, we model the conformal scores of sequences that share the same token at step $l$ as an empirical distribution and group tokens with similar distributions into the same local cluster.  This is reasonable since tokens with similar empirical score distributions often reflect comparable semantic uncertainty, and clustering them enables more stable quantile estimation. For instance, high-confidence clusters with sharp score distributions require stricter thresholds for search efficiency, while low-confidence or long-tail clusters with flatter distributions may need looser thresholds to avoid under-coverage.

Formally, at each step \( l \), given a random prefix sequence \( \mathbf{S}^{1:l-1} \), we consider its one-token continuations \( \{\mathbf{S}^{1:l-1}a\}_{a \in \mathcal{Y}} \). Since  the prefix \( \mathbf{S}^{1:l-1} \) are stochastic, each token \( a \in \mathcal{Y} \) at step $l$ induces a conditional distribution over input–prefix pairs, denoted by $\sigma(X, \mathbf{S}^{1:l-1}a) \,\big|\, S^l = a$.
This distribution captures the model's belief about the plausibility of contexts that lead to the selection of token \( a \) at step \( l \), reflecting variability over the distribution of input–prefix paths conditioned on \( S^l = a \). 
For instance, if this distribution is sharply concentrated over a narrow set of contexts, it indicates that the model selects token \( a \) consistently and with high confidence. We then cluster the tokens at step $l$ according to the similarity of their score distributions:
\[
h^{*}_{l}(a)=h^{*}_{l}(a')
\;\;\Longleftrightarrow\;\;
{\rm Divergence}\!\left(\sigma\!\bigl(x,\mathbf{S}^{1:l-1}a\bigr),
          \sigma\!\bigl(x,\mathbf{S}^{1:l-1}a'\bigr)\right)\le\delta,
\]
where \(h^{*}_{l}:\mathcal{Y}\to\mathcal{M}_{l}\) is the oracle clustering map,  
\({\rm Divergence}(\cdot,\cdot)\) is a statistical divergence, and \(\delta\) is a tolerance threshold.  Each $l$-th step \(S^l\) is assigned to a unique cluster
\(m=h^{*}_{l}(S^{l})\in\mathcal{M}_{l}\).


Building on the above discussion, we  can define the step-cluster failure event by
\[
E_{l,m}:=\Bigl\{
h^{*}_{l}(S^{l})=m,\;
\mathbf{S}^{1:l}\notin\mathcal{C}^{(l)}(X;\beta),\;
\mathbf{S}^{1:k}\in\mathcal{C}^{(k)}(X;\beta)\;\forall\,k<l
\Bigr\}.
\]
The collection \(\{E_{l,m}\}_{l,m}\) forms a disjoint partition of the full non‑coverage event.  Consequently,
\[
\mathbb{P}\!\bigl[\mathbf{S}\notin\mathcal{C}^{(L)}(X;\beta)\bigr]
\;=\;
\sum_{l,m}
\mathbb{P}\!\bigl[h^{*}_{l}(S^{l})=m\bigr]\,
\mathbb{P}\!\bigl[(X,\mathbf{S}^{1:l})\in E_{l,m}\mid h^{*}_{l}(S^{l})=m\bigr].
\]
\begin{remark}
This decomposition highlights how each cluster-step pair contributes to the overall non-coverage risk. The probability \( \mathbb{P}\bigl[h^{*}_{l}(S^{l}) = m\bigr] \) captures how often a token falls into cluster \( m \) at step \( l \), while the conditional term quantifies the non-coverage rate within that pair. Under uniform quantile calibration like previous conformal prediction methods, high-frequency clusters dominate the global objective, while long-tail clusters have little influence despite potentially high error rates. As a result, conformal sets calibrated this way often sacrifice coverage in low-frequency regions. This motivates our cluster-step–aware calibration strategy.
\end{remark}

This decomposition motivates the design of \textsc{CoVeR}, which consists of three key components (as shown in Figure \ref{fig:enter-label}):  \textit{conformal score computation}, \textit{clustering learning} and \textit{quantile-threshold estimation}, both of which are carried out using a calibration dataset \( D = \{(x_i, \mathbf{s}_i)\}_{i=1}^N\).  We begin by randomly splitting the calibration dataset of size \( N \) into two disjoint parts: the \textit{clustering dataset} \( D_1 = \{(x_i, \mathbf{s}_i) : i \in \mathcal{I}_1\} \), used for performing class and sequence clustering, and the \textit{proper calibration dataset} \( D_2 = \{(x_i, \mathbf{s}_i) : i \in \mathcal{I}_2\} \), used for computing conformal quantiles. The index sets satisfy \( |\mathcal{I}_1| = \lfloor \gamma N \rfloor \) and \( |\mathcal{I}_2| = N - |\mathcal{I}_1| \), for some tuning parameter \( \gamma \in [0, 1] \).

\begin{figure}
    \centering
    \includegraphics[width=1\linewidth]{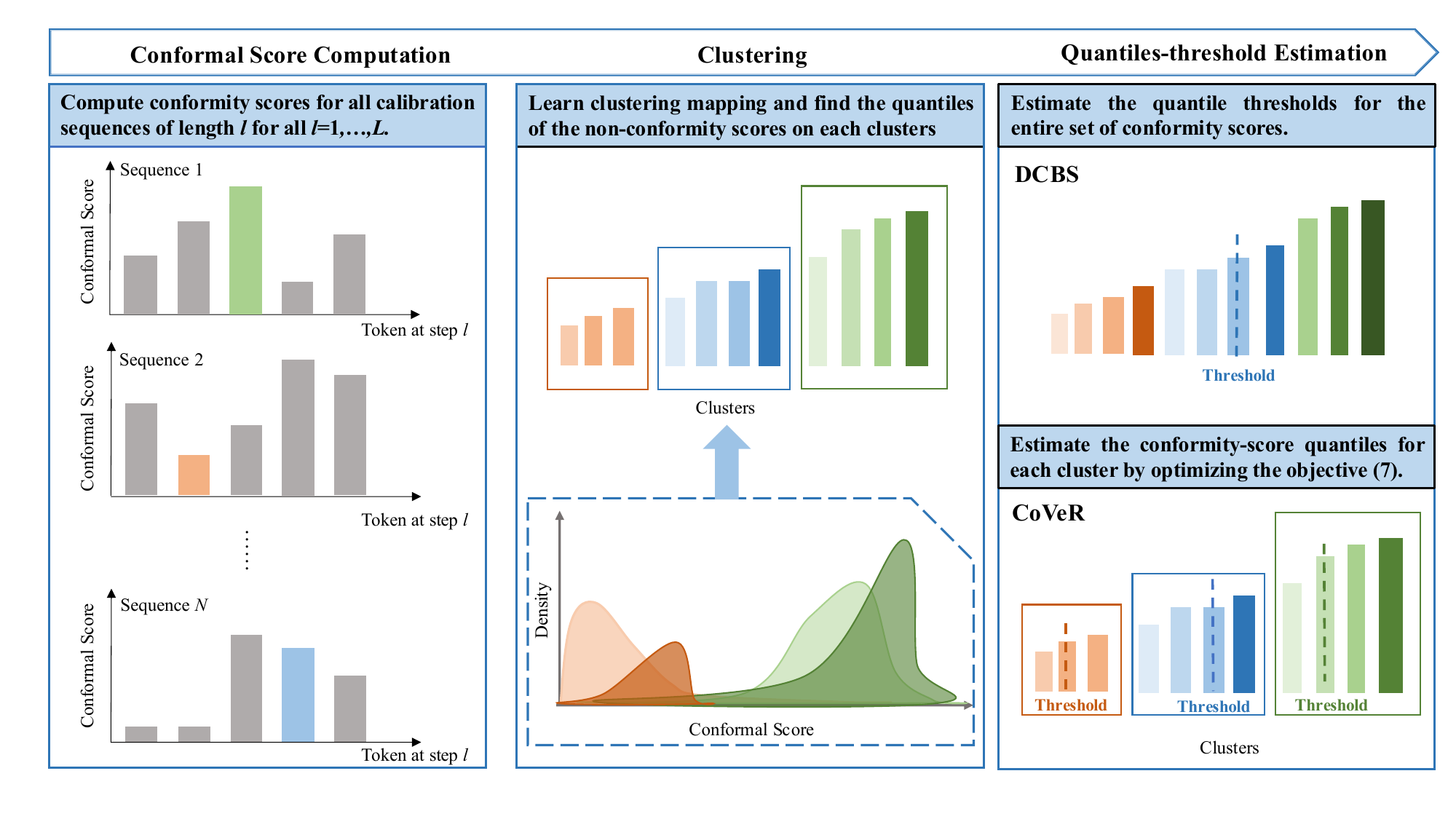}
    \caption{Construction process of  \textsc{CoVeR}.
First, we compute the conformal scores of all length-$l$ calibration sequences for constructing quantile embeddings in Eq. \eqref{eq_quantileembedding}, where the colored segments represent the labels appearing in the sequence at step $l$. Notably, some labels, such as the orange ones, belong to long-tail categories. Second, we treat the conformal scores of all length-$l$ sequences that share a same token at step $l$ as samples from that token’s empirical distribution, and tokens whose distributions have similar shapes are then clustered together (see Eq. \eqref{eq_cluster}). Finally, \textsc{CoVeR} optimizes the objective in~Eq.\eqref{objective} to learn a separate threshold for each cluster, thereby retaining credible portions of long-tail classes (right, lower panel). By contrast, in the right and upper panel, DCBS proposed in \citep{deutschmann2024conformal} applies a single global threshold (i.e., the (1-$\alpha$)-quantile of all conformity scores), which tends to exclude long-tail classes.}
    \label{fig:enter-label}
\end{figure}

\subsection{Distribution-aware Class Clustering}\label{sect_clustering}
To identify class-step cluster with similar uncertainty patterns, we construct quantile embeddings for each class at each step by aggregating empirical score distributions.  Formally,  for each $(y,l)$ pair, the quantile embedding vector \( \mathbf{z}^{y,l} \in \mathbb{R}^{|\mathcal{T}|} \) is defined as:
\begin{equation}\label{eq_quantileembedding}
\mathbf{z}^{y,l} := \left(\mathrm{Quantile}\left(\tau, \{ r_i^{1:l} \mid s_i^l = y \}_{i \in \mathcal{I}_1} \right) \right)_{\tau \in \mathcal{T}},
\end{equation}
where \( \mathcal{T} \subset [0, 1] \) denotes a predefined set of quantile levels (e.g., $\mathcal{T}=(0.5,0.6,0.7,0.8,0.9)$), and \( r_i^{1:l} =\sigma(x_i,\mathbf{s}_i^{1:l})\) represents the conformal score of the sequence $\mathbf{s}_i^{1:l}$. For instance, this metric can be chosen as  the conditional probability mass function
$$
\sigma(X,\mathbf{S}^{1:l}) = p(S^1\mid X) \prod_{k=1}^{l} p(S^k \mid X S^1 \cdots S^{k-1}),~ \forall l \in [L].
$$

For each step $l$, we then perform a weighted \( k \)-means clustering over the set of valid embeddings \( \{ \mathbf{z}^{y,l} \}_{y\in |\mathcal{Y}|} \), using weights \( |\mathcal{I}_1^{y,l}|^{1/2} \) to reflect the reliability of each embedding, resulting in a set of clustering assignment functions:
\begin{equation}\label{eq_cluster}
\hat{h}_l : \mathcal{Y}  \rightarrow \mathcal{M}:= \{1, \dots, M\} \cup \{\texttt{null}\}, \forall l\in[L].
\end{equation}

To ensure reliable quantile estimation, inspired by \citep{ding2023class}, we assign any class-step pair \( (y,l) \) to a designated \texttt{null} cluster if its number of calibration instances falls below a threshold. This prevents unstable or overly noisy quantile estimates due to insufficient data. All \texttt{null} cluster members are grouped together and share a global quantile threshold, ensuring they are still covered under the conformal guarantee while avoiding unreliable local calibration. 
\begin{remark}
In practice, since conformity score distributions vary smoothly across nearby steps, it is unnecessary to calibrate thresholds at each individual position. To address sample sparsity at deeper steps, we adopt a step-bucketed clustering strategy: step indices are grouped into fixed-width buckets (e.g., steps 1–5, 6–10),  calibration samples are aggregated within each bucket before  clustering, and cluster-specific quantiles are shared across all steps in the same bucket. This approach enables stable clustering and reliable quantile estimation of \textsc{CoVeR}.
\end{remark}

\subsection{Optimization Objective for CoVeR}\label{sect_optimizationobjective}
 Denote by \( \mathcal{I}_2^{y,l} = \{i : S_i^l = y\} \) the set of indices corresponding to calibration examples for which the label at step \( l \) is \( y \). Let \( \mathcal{I}_2^l(m) = \bigcup_{\hat{h}_l(y) = m} \mathcal{I}_2^{y,l} \) denote the set of indices in the calibration dataset \( D_2 \) whose step-\( l \) class labels are assigned to cluster \( m \) under the clustering function \( \hat{h}_l \).

Given a set of learnable quantile levels 
$$
\beta = (\beta_{1,1},\cdots,\beta_{\text{null},1},\cdots,\beta_{1,L},\cdots,\beta_{\text{null},L})^T\in\mathbb{R}^{(M+1)L}
$$
for cluster set $\mathcal{M}_l=\{1, \dots, M, \text{null}\}$,  we  define the full-path coverage indicator for  \( i \)-th sample as:
\begin{equation}
\mathcal{L}(x_i,\mathbf{s}_i;\beta) :=\mathbb{I}[\mathbf{s}_i\in \mathcal C^{(L)}(x_i;\beta)]=\prod_{l=1}^L\mathbb{I}[\mathbf{s}_i^{1:l}\in \mathcal C^{(l)}(x_i;\beta)] = \prod_{l=1}^L \mathbb{I}\left\{ \sigma(x_i, \mathbf{s}_i^{1:l}) \geq Q_l(\hat{h}_l(s_i^l);\beta) \right\},
\end{equation}
where the conformal set is 
\begin{equation*}
\mathcal{C}^{(l+1)}(X;\beta) = \left\{
\mathbf S^{1:l} a ~ \middle| ~ \sigma(X, \mathbf{S}^{1:l}a) \geq Q_{l+1}(\hat{h}_{l+1}(a);\beta),~ a \in \mathcal{Y}, ~\ \mathbf{S}^{1:l} \in \mathcal{C}^{(l)}(X;\beta)\right\},
\end{equation*}
and the conformal quantiles are 
\[
Q_l(m;\beta) =  \text{Quantile} \left(   \frac{|\mathcal{I}^l_2(m)| + 1) \beta_{m,l}}{|\mathcal{I}^l_2(m)|} , \{\sigma(x_i, \mathbf{s}_i^{1:l})\}_{i \in \mathcal{I}^l_2(m)} \right)
\]
and 
\[
Q_l(\text{null};\beta) =  \text{Quantile} \left(   \frac{(|\mathcal{I}^l_2| + 1)\beta_{{\rm null},l}}{|\mathcal{I}^l_2|} , \{\sigma(x_i, \mathbf{s}_i^{1:l})\}_{i \in \mathcal{I}_2} \right).
\]
Intuitively, each \( \mathcal L(x_i,\mathbf{s}_i;\beta)\) equals $1$ if  the prediction set that covers the full-path trajectory \( \mathbf{s}_i \).  Here, constraining the full-path coverage rate is essential, as it reflects whether the entire reasoning trajectory remains valid and correct. This global constraint directly aligns with the ultimate task goal. However, it still tends to prioritize high-density sample regions, since these dominate the empirical average. As a result, cluster-step pairs with low sample counts may exhibit high noncoverage but contribute little to the total risk, and thus be ignored during optimization.

To alleviate this problem, we encourage the model to maintain calibration across all local regions by introducing a regularization term over the cluster-step noncoverage rates. Formally,  we formulate the  optimization problem for \textsc{CoVeR}:
\begin{equation}\label{objective}
\max_{\beta}  \quad  \sum_{m \in \mathcal{M}}\sum_{l \in [L]} \left[Q_l(m;\beta)-\lambda_{l,m}\hat{\epsilon}_{l,m}\right], \quad
\text{s.t.}  \quad \frac{1}{|\mathcal{I}_2|} \sum_{i \in \mathcal{I}_2}  \mathcal{L}(x_i,\mathbf{s}_i;\beta) \geq 1 - \alpha,
\end{equation}
where 
$$
\hat{\epsilon}_{l,m} := \frac{1}{|\mathcal{I}_2^l(m)|} \sum_{i\in \mathcal{I}_2^l(m)} \mathbb{I}[\sigma(x_i, \mathbf s_i^{1:l}) < Q_l(m;\beta) ]\prod_{k<l}\mathbb{I}[\sigma(x_i, \mathbf s_i^{1:k}) > Q_l(\hat{h}_k(s_i^k);\beta)]
$$
is the empirical error for $(l,m)$ pair, and the regularization parameter $\lambda_{l,m}$ specifies the weight assigned to controlling the non-coverage rate of the $(l,m)$ step-cluster pair. Given the estimated quantile levels \( \hat{\beta} \) and the cluster-specific quantiles \( Q_l(m;\hat{\beta}), \forall l\in [L-1] \), we establish the conformal prediction set at inference time for the \( l \)-th reasoning step as follows:
\begin{equation}
\hat{\mathcal{C}}^{(l)}(X) = \left\{
\mathbf{S}^{1:l-1} a ~ \middle| ~ \sigma(X, \mathbf{S}^{1:l-1}a) \geq Q_{l}(\hat{h}_{l}(a);\hat \beta),~ a \in \mathcal{Y}, ~\ \mathbf{S}^{1:l-1} \in \hat{\mathcal{C}}^{(l-1)}(X)\right\}.
\end{equation}
\begin{remark}
Conventional class-conditional conformal prediction methods assign the same confidence level $\alpha$ to the conformal set of each class \citep{ding2023class}. However, applying a uniform quantile across all clusters and steps can result in overly conservative prediction sets for easy clusters and undercoverage for harder ones.  In the context of multi-step autoregressive next-token prediction, we instead allow conformal sets for different clusters to have heterogeneous confidence levels. This adaptive calibration is crucial for search efficiency: easier step-cluster pairs can tolerate tighter thresholds through larger $\beta_{l,m}$, while harder step-cluster pairs  may require more relaxed ones.
\end{remark}


\subsection{Optimization Algorithm}\label{sect_optimization}

We formulate the objective problem in Eq. (\ref{objective}) as a constrained optimization over the matrix \( \beta \in [0, 1]^{(M+1) \times L} \), where each entry \( \beta_{m,l} \) determines the quantile used to compute the conformal threshold \( Q_l(m; \beta) \). The objective is to minimize the total size of the prediction set, approximated as maximizing \( \sum_{m,l} Q_l(m;\beta) \), while ensuring that the full-path coverage rate across validation samples is no less than \( 1 - \alpha \). 

To efficiently solve this optimization problem without expensive grid search, we adopt a greedy trade-off algorithm utilized in \citep{zheng2024optimizing} that exploits the convexity and sparsity of the feasible region. We first initialize all entries in \( \beta \) to $0$, except for a selected entry \( \beta_{k,j} \), which is initialized to $1$. This effectively collapses the prediction set to focus on a single cluster-step pair. We then gradually decrease \( \beta_{k,j} \) until the full-path coverage constraint is just satisfied. Once coverage is met, we iteratively improve the efficiency by selecting random cluster \( (m,l) \neq (k,j) \) and performing a trade-off: we increment \( \beta_{k,j} \) slightly and decrement \( \beta_{m,l} \) to maintain coverage. If the new configuration results in a larger objective value $\sum_{m} \sum_{l} \left[ Q_l(m; \beta) - \lambda_{l,m} \hat{\epsilon}_{l,m} \right]$, we accept the update, as it indicates a more compact overall prediction set while still satisfying the desired coverage constraints.

To ensure stability and avoid infinite loops, we impose a finite trade-off budget \( B \), and the trade-off increment \( \epsilon \) is typically set to a small value like \( \epsilon = \frac{1}{|\mathcal{I}_2|} \). When increasing \( \beta_{m, l} \), we make sure that coverage is maintained throughout. This strategy effectively traces a step-wise boundary along the feasible region of valid \( \beta \)-configurations and converges to a highly efficient, approximately optimal quantile assignment \( \hat{\beta} \).

\begin{algorithm}[H]
\caption{Efficient Optimization for \textsc{CoVeR}\label{alg-cover}}
\KwIn{Calibration set $\{(x_i, \mathbf{s}_i)\}_{i \in \mathcal{I}_2}$, non-conformity scores $\{r_i^{1:l}\}_{i \in \mathcal{I}_2, l \in [L]}$, step quantile levels $\mathcal{T}$, initial coordinate $(k,j)$, budget $B$, increment $\epsilon$}
\KwOut{Adaptive quantile levels $\hat{\beta} \in [0,1]^{(M+1) \times L}$}
\vspace{0.6em}
\textbf{Step 1: Quantile Embedding and Clustering} \\
\ForEach{$l \in [L]$}{
    Compute quantile embeddings $\{ \mathbf{z}_i^{l} \}_{i \in \mathcal{I}_1}$ for each class $y \in \mathcal{Y}$: \\
    \Indp $\mathbf{z}_i^{l} \leftarrow \text{Quantile}(\{r_j^{1:l} : s_j^{l} = y \}, \tau \in \mathcal{T})$\\
    Perform weighted $k$-means clustering using embeddings and weights $|\mathcal{I}_y^{l}|^{1/2}$ \\
    Assign cluster labels via $\hat{h}_{l}: \mathcal{Y} \rightarrow \{1, \dots, M\} \cup \{\text{null}\}$ 
}
\vspace{0.6em}
\textbf{Step 2: Quantile Initialization and Optimization} \\
Initialize: $\beta_{l,m} \leftarrow 0$ for all $(l, m) \in [L] \times \mathcal{M} \setminus \{(k,j)\}$, and $\beta_{k,j} \leftarrow 1$ \\
\While{$\frac{1}{|\mathcal{I}_2|} \sum_{i \in \mathcal{I}_2} \mathcal{L}(x_i, s_i; \beta) < 1 - \alpha + \frac{1}{|\mathcal{I}_2|}$}{
\vspace{0.3em}
    $\beta_{k,j} \leftarrow \beta_{k,j} - \frac{1}{|\mathcal{I}_2|}$
\vspace{0.3em}
}

\For{$\text{iter} \leftarrow 1$ \KwTo $B$}{
    Randomly sample $(m,l) \in \mathcal{M} \times [L]$ with $(m,l) \neq (k,j)$ \\
    Set trial quantiles $\beta'_{k,j} \leftarrow \beta_{k,j}$, $\beta'_{m,l} \leftarrow \beta_{m,l}$ \\
    Let $\beta' \leftarrow \beta$ with entries $(k,j)$ and $(m,l)$ replaced by $\beta'_{k,j}$ and $\beta'_{m,l}$ \\
    \While{$\frac{1}{|\mathcal{I}_2|} \sum_{i \in \mathcal{I}_2} \mathcal{L}(x_i, s_i; \beta') > 1 - \alpha$}{
    \vspace{0.3em}
    $\beta'_{m,l} \leftarrow \beta'_{m,l}+\epsilon$
    \vspace{0.3em}
    }
    \If{$\sum_{m} \sum_{l} \left[ Q_l(m; \beta') - \lambda_{l,m} \hat{\epsilon}'_{l,m} \right] > \sum_{m} \sum_{l} \left[ Q_l(m; \beta) - \lambda_{l,m} \hat{\epsilon}_{l,m} \right]$}{
    \vspace{0.3em}
            Update: $\beta_{k,j} \leftarrow \beta'_{k,j},\quad \beta_{m,l} \leftarrow \beta'_{m,l}$
    \vspace{0.3em}
    }
}
\Return $\hat{\beta} \leftarrow \beta$
\end{algorithm}

\subsection{Probably Approximately Correct Guarantees}\label{sect_pacbound}

Our proposed \textsc{CoVeR} takes as input a validation set $Z_{\text{val}} \subseteq \mathcal{Z}$ consisting of $|\mathcal{I}_{2}|$ i.i.d. samples, and outputs a confidence set predictor $\hat{\mathcal C}^{(L)}$. Given $\epsilon, \delta \in \mathbb{R}_{>0}$, we say that \textsc{CoVeR}  is \textit{probably approximately correct (PAC)} if
\begin{equation}\label{pac}
\mathbb{P}_{Z_{\text{val}} \sim \mathcal{D}^n} \left[\mathbb{P}_{(X,\mathbf S)\sim \mathcal{D}}\left[ \mathbb{I}[\mathbf S\in \hat{\mathcal C}^{(L)}(X)]\right]  > 1-\alpha \right] > 1-\delta.
\end{equation}

To establish such a PAC-style generalization guarantee for \textsc{CoVeR}, we address two main challenges. First, because our algorithm enforces coverage at the level of cluster-step pairs, we must decompose the full-sequence non-coverage event into interpretable local components (see Section~\ref{sect_decomposition}). This decomposition enables us to trace global failure back to its localized sources, revealing the contribution of each cluster and reasoning step. Second, following the approach in \citep{park2019pac}, we reformulate the local coverage estimation problem as a binary classification task. This transformation allows us to leverage classical tools from statistical learning theory, such as concentration inequalities, to derive high-confidence upper bounds on the true non-coverage rates (Section~\ref{sect_pac}).

\begin{theorem}
Assume the training and testing data are independently and identically distributed (i.i.d.), and that  clustering algorithm $\hat h_l, \forall l\in[L]$  successfully partition the space of class-step pairs $(y,l)$ into distinct clusters. Denote the empirical cluster probabilities and empirical noncoverage rates by
\[
\hat{p}_{l,m} := \frac{1}{N} \sum_{i=1}^N \mathbb{I}[\hat{h}_l(S_i^l) = m].
\]
Then, with probability at least \( 1 - \delta - \zeta \), where $\delta + \zeta := \sum_{l,m} \delta_{l,m} + \zeta_{l,m}, \forall \delta_{l,m}\in(0,1), \forall \zeta_{l,m}\in(0,1)$,  the full-path failure probability satisfies:
\begin{eqnarray*}
\mathbb{P}_{(X,\mathbf S)\sim\mathcal{D}} \left[\mathbf S \notin \hat{\mathcal{C}}^{(L)}(X) \right]
&\leq&
\alpha + \sum_{l,m} \hat{p}_{l,m}
\left(
\sqrt{ \frac{2\hat{v}_{l,m} \log(3/\delta_{l,m})}{n_{l,m}} }
+ \frac{3\log(3/\delta_{l,m})}{n_{l,m}}
\right)\\
&+& \sum_{l,m} \hat{\epsilon}_{l,m} \sqrt{ \frac{ \log(2/\zeta_{l,m}) }{ 2n_{l,m} } },
\end{eqnarray*}
where \( \hat{v}_{l,m} \) is the empirical variance of the indicator \( \mathbb{I}[(x_i, \mathbf s_i^{1:l}) \in E_{l,m}] \). 
\end{theorem}

\begin{remark}
This bound offers the following key theoretical insights: 1) It explicitly quantifies the contribution of each cluster-step pair \( (l,m) \) to the desirable full-path failure probability, enabling localized understanding and diagnosis of failure behavior;  2) Unlike the exponentially decaying full-path coverage probability (e.g., $(1 - \alpha)^L$ in \citep{deutschmann2024conformal}), our bound accounts for the optimization framework, allowing us to achieve a coverage rate of $1 - \alpha$ that is independent of the next-token prediction length $L$.
\end{remark}




\section{Conclusion}
This paper introduces \textsc{CoVeR}, a novel model-free decoding strategy grounded in the conformal prediction and dual-objective optimization framework. \textsc{CoVeR} simultaneously maintains a compact search space and guarantees high coverage probability over desirable generation trajectories.   Theoretically, we establish a PAC-style generalization bound that guarantees asymptotic coverage at any desired confidence level. 

\bibliography{Refs}

\begin{thebibliography}{32}
\providecommand{\natexlab}[1]{#1}
\providecommand{\url}[1]{\texttt{#1}}
\expandafter\ifx\csname urlstyle\endcsname\relax
  \providecommand{\doi}[1]{doi: #1}\else
  \providecommand{\doi}{doi: \begingroup \urlstyle{rm}\Url}\fi

\bibitem[Angelopoulos et~al.(2021)Angelopoulos, Bates, Cand{\`e}s, Jordan, and Lei]{angelopoulos2021learn}
Anastasios~N Angelopoulos, Stephen Bates, Emmanuel~J Cand{\`e}s, Michael~I Jordan, and Lihua Lei.
\newblock Learn then test: Calibrating predictive algorithms to achieve risk control.
\newblock \emph{arXiv preprint arXiv:2110.01052}, 2021.

\bibitem[Arnaout et~al.(2022)Arnaout, Razniewski, Weikum, and Pan]{arnaout2022uncommonsense}
Hiba Arnaout, Simon Razniewski, Gerhard Weikum, and Jeff~Z Pan.
\newblock Uncommonsense: Informative negative knowledge about everyday concepts.
\newblock In \emph{Proceedings of the 31st ACM International Conference on Information \& Knowledge Management}, pages 37--46, 2022.

\bibitem[Balasubramanian et~al.(2014)Balasubramanian, Ho, and Vovk]{balasubramanian2014conformal}
Vineeth Balasubramanian, Shen-Shyang Ho, and Vladimir Vovk.
\newblock \emph{Conformal prediction for reliable machine learning: theory, adaptations and applications}.
\newblock Newnes, 2014.

\bibitem[Chan et~al.(2025)Chan, Cheng, Huang, Chen, and Huang]{chan2025efficient}
Brian~J Chan, Jui-Hung Cheng, Mao~Xun Huang, Chao-Ting Chen, and Hen-Hsen Huang.
\newblock Efficient beam search for large language models using trie-based decoding.
\newblock \emph{arXiv preprint arXiv:2502.00085}, 2025.

\bibitem[Chen et~al.(2023)Chen, Shi, Fu, Cheng, Li, and Xiao]{chen2023say}
Jiangjie Chen, Wei Shi, Ziquan Fu, Sijie Cheng, Lei Li, and Yanghua Xiao.
\newblock Say what you mean! large language models speak too positively about negative commonsense knowledge.
\newblock In \emph{Proceedings of the Annual Meeting of the Association for Computational Linguistics}, pages 9890--9908. Association for Computational Linguistics, 2023.

\bibitem[Chiang and Chen(2021)]{chiang2021relating}
Ting-Rui Chiang and Yun-Nung Chen.
\newblock Relating neural text degeneration to exposure bias.
\newblock In \emph{Proceedings of the Fourth BlackboxNLP Workshop on Analyzing and Interpreting Neural Networks for NLP}, pages 228--239. Association for Computational Linguistics, 2021.

\bibitem[Deutschmann et~al.(2024)Deutschmann, Alberts, and Mart{\'\i}nez]{deutschmann2024conformal}
Nicolas Deutschmann, Marvin Alberts, and Mar{\'\i}a~Rodr{\'\i}guez Mart{\'\i}nez.
\newblock Conformal autoregressive generation: Beam search with coverage guarantees.
\newblock In \emph{Proceedings of the AAAI Conference on Artificial Intelligence}, 2024.

\bibitem[Ding et~al.(2023)Ding, Angelopoulos, Bates, Jordan, and Tibshirani]{ding2023class}
Tiffany Ding, Anastasios Angelopoulos, Stephen Bates, Michael Jordan, and Ryan~J Tibshirani.
\newblock Class-conditional conformal prediction with many classes.
\newblock In \emph{Advances in neural information processing systems}, pages 64555--64576, 2023.

\bibitem[Guo et~al.(2025)Guo, Yang, Zhang, Song, Zhang, Xu, Zhu, Ma, Wang, Bi, et~al.]{guo2025deepseek}
Daya Guo, Dejian Yang, Haowei Zhang, Junxiao Song, Ruoyu Zhang, Runxin Xu, Qihao Zhu, Shirong Ma, Peiyi Wang, Xiao Bi, et~al.
\newblock Deepseek-r1: Incentivizing reasoning capability in llms via reinforcement learning.
\newblock \emph{arXiv preprint arXiv:2501.12948}, 2025.

\bibitem[Huang et~al.(2023)Huang, Gu, Hou, Wu, Wang, Yu, and Han]{huang-etal-2023-large}
Jiaxin Huang, Shixiang Gu, Le~Hou, Yuexin Wu, Xuezhi Wang, Hongkun Yu, and Jiawei Han.
\newblock Large language models can self-improve.
\newblock In \emph{Proceedings of the Conference on Empirical Methods in Natural Language Processing}, pages 1051--1068, 2023.

\bibitem[Jaech et~al.(2024)Jaech, Kalai, Lerer, Richardson, El-Kishky, Low, Helyar, Madry, Beutel, Carney, et~al.]{jaech2024openai}
Aaron Jaech, Adam Kalai, Adam Lerer, Adam Richardson, Ahmed El-Kishky, Aiden Low, Alec Helyar, Aleksander Madry, Alex Beutel, Alex Carney, et~al.
\newblock Openai o1 system card.
\newblock \emph{arXiv preprint arXiv:2412.16720}, 2024.

\bibitem[Jones and Bergen(2025)]{jones2025large}
Cameron~R Jones and Benjamin~K Bergen.
\newblock Large language models pass the turing test.
\newblock \emph{arXiv preprint arXiv:2503.23674}, 2025.

\bibitem[Li et~al.(2024)Li, Liao, Ning, Wang, Li, Lu, Brahman, Zhao, Choi, and Ren]{li2023search}
Huihan Li, Zeyi Liao, Yuting Ning, Siyuan Wang, Xiang~Lorraine Li, Ximing Lu, Faeze Brahman, Wenting Zhao, Yejin Choi, and Xiang Ren.
\newblock In search of the long-tail: Systematic generation of long-tail knowledge via logical rule induced search, 2024.

\bibitem[Minsky(1997)]{minsky1997negative}
Marvin Minsky.
\newblock Negative expertise.
\newblock 1997.

\bibitem[OpenAI(2025)]{o3_openai}
OpenAI.
\newblock Openai o3 system card.
\newblock \emph{Technical Report}, 2025.

\bibitem[Park et~al.(2020)Park, Bastani, Matni, and Lee]{park2019pac}
Sangdon Park, Osbert Bastani, Nikolai Matni, and Insup Lee.
\newblock Pac confidence sets for deep neural networks via calibrated prediction.
\newblock In \emph{Proceedings of the International Conference on Learning Representations}, 2020.

\bibitem[Plassier et~al.(2025)Plassier, Fishkov, Guizani, Panov, and Moulines]{zheng2024optimizing}
Vincent Plassier, Alexander Fishkov, Mohsen Guizani, Maxim Panov, and Eric Moulines.
\newblock Probabilistic conformal prediction with approximate conditional validity.
\newblock In \emph{Proceedings of the International Conference on Learning Representations}, 2025.

\bibitem[Quach et~al.(2024)Quach, Fisch, Schuster, Yala, Sohn, Jaakkola, and Barzilay]{quach2023conformal}
Victor Quach, Adam Fisch, Tal Schuster, Adam Yala, Jae~Ho Sohn, Tommi~S. Jaakkola, and Regina Barzilay.
\newblock Conformal language modeling.
\newblock In \emph{Proceedings of the International Conference on Learning Representations}, 2024.

\bibitem[Ravfogel et~al.(2023)Ravfogel, Goldberg, and Goldberger]{ravfogel2023conformal}
Shauli Ravfogel, Yoav Goldberg, and Jacob Goldberger.
\newblock Conformal nucleus sampling.
\newblock In \emph{Findings of the Association for Computational Linguistics}, pages 27--34, 2023.

\bibitem[Ren et~al.(2023)Ren, Dixit, Bodrova, Singh, Tu, Brown, Xu, Takayama, Xia, Varley, Xu, Sadigh, Zeng, and Majumdar]{ren2023robots}
Allen~Z. Ren, Anushri Dixit, Alexandra Bodrova, Sumeet Singh, Stephen Tu, Noah Brown, Peng Xu, Leila Takayama, Fei Xia, Jake Varley, Zhenjia Xu, Dorsa Sadigh, Andy Zeng, and Anirudha Majumdar.
\newblock Robots that ask for help: Uncertainty alignment for large language model planners.
\newblock In \emph{7th Annual Conference on Robot Learning}, 2023.

\bibitem[Shafer and Vovk(2008)]{shafer2008tutorial}
Glenn Shafer and Vladimir Vovk.
\newblock A tutorial on conformal prediction.
\newblock \emph{Journal of Machine Learning Research}, 9\penalty0 (3), 2008.

\bibitem[Sridhar and Visser(2022)]{sridhar2022improved}
Arvind~Krishna Sridhar and Erik Visser.
\newblock Improved beam search for hallucination mitigation in abstractive summarization.
\newblock \emph{arXiv preprint arXiv:2212.02712}, 2022.

\bibitem[Sun et~al.(2023)Sun, Zheng, Xie, Liu, Chu, Qiu, Xu, Ding, Li, Geng, et~al.]{sun2023survey}
Jiankai Sun, Chuanyang Zheng, Enze Xie, Zhengying Liu, Ruihang Chu, Jianing Qiu, Jiaqi Xu, Mingyu Ding, Hongyang Li, Mengzhe Geng, et~al.
\newblock A survey of reasoning with foundation models.
\newblock \emph{arXiv preprint arXiv:2312.11562}, 2023.

\bibitem[Tang et~al.(2023)Tang, Peng, Wang, Ding, Durrett, and Rousseau]{tang2023less}
Liyan Tang, Yifan Peng, Yanshan Wang, Ying Ding, Greg Durrett, and Justin~F Rousseau.
\newblock Less likely brainstorming: Using language models to generate alternative hypotheses.
\newblock In \emph{Proceedings of the Annual Meeting of the Association for Computational Linguistics}, page 12532, 2023.

\bibitem[Tibshirani et~al.(2019)Tibshirani, Foygel~Barber, Candes, and Ramdas]{tibshirani2019conformal}
Ryan~J Tibshirani, Rina Foygel~Barber, Emmanuel Candes, and Aaditya Ramdas.
\newblock Conformal prediction under covariate shift.
\newblock In \emph{Advances in neural information processing systems}, 2019.

\bibitem[Wang et~al.(2022)Wang, Kordi, Mishra, Liu, Smith, Khashabi, and Hajishirzi]{wang2022selfa}
Yizhong Wang, Yeganeh Kordi, Swaroop Mishra, Alisa Liu, Noah~A Smith, Daniel Khashabi, and Hannaneh Hajishirzi.
\newblock Self-instruct: Aligning language models with self-generated instructions.
\newblock \emph{arXiv preprint arXiv:2212.10560}, 2022.

\bibitem[Wei et~al.(2022)Wei, Wang, Schuurmans, Bosma, Xia, Chi, Le, and Zhou]{Wei-NeurIPS-models-2022}
Jason Wei, Xuezhi Wang, Dale Schuurmans, Maarten Bosma, Fei Xia, Ed~Chi, Quoc~V Le, and Denny Zhou.
\newblock Chain-of-thought prompting elicits reasoning in large language models.
\newblock In \emph{Advances in Neural Information Processing Systems}, 2022.

\bibitem[Xie et~al.(2023)Xie, Kawaguchi, Zhao, Zhao, Kan, He, and Xie]{xie2023self}
Yuxi Xie, Kenji Kawaguchi, Yiran Zhao, James~Xu Zhao, Min-Yen Kan, Junxian He, and Michael Xie.
\newblock Self-evaluation guided beam search for reasoning.
\newblock In \emph{Advances in Neural Information Processing Systems}, pages 41618--41650, 2023.

\bibitem[Yang et~al.(2024)Yang, Lee, and Tadepalli]{yang2024language}
Yilin Yang, Stefan Lee, and Prasad Tadepalli.
\newblock Language-informed beam search decoding for multilingual machine translation.
\newblock In \emph{Findings of the Association for Computational Linguistics}, pages 15761--15772, 2024.

\bibitem[Zhao et~al.(2023)Zhao, Khalman, Joshi, Narayan, Saleh, and Liu]{zhao2022calibrating}
Yao Zhao, Mikhail Khalman, Rishabh Joshi, Shashi Narayan, Mohammad Saleh, and Peter~J Liu.
\newblock Calibrating sequence likelihood improves conditional language generation.
\newblock In \emph{Proceedings of the International Conference on Learning Representations}, 2023.

\bibitem[Zhou et~al.(2021)Zhou, Khanna, Lee, Lin, Ho, Pujara, and Ren]{zhou2020rica}
Pei Zhou, Rahul Khanna, Seyeon Lee, Bill~Yuchen Lin, Daniel Ho, Jay Pujara, and Xiang Ren.
\newblock {RICA}: Evaluating robust inference capabilities based on commonsense axioms.
\newblock In \emph{Proceedings of the 2021 Conference on Empirical Methods in Natural Language Processing}, pages 7560--7579, 2021.

\bibitem[Zhu et~al.(2024)Zhu, Zhang, Xie, and Su]{zhu2024deductive}
Tinghui Zhu, Kai Zhang, Jian Xie, and Yu~Su.
\newblock Deductive beam search: Decoding deducible rationale for chain-of-thought reasoning.
\newblock In \emph{First Conference on Language Modeling}, 2024.

\end{thebibliography}
\bibliographystyle{plainnat}

\medskip

\appendix
\newpage

\begin{center}
    \LARGE \textbf{Appendix}
\end{center}

\section{Proof of Theorem 1}

\subsection{Non-coverage Rate Decomposition}\label{sect_decomposition}
To establish the PAC bound in Eq.~\eqref{pac}, we begin by decomposing the full-path non-coverage event into cluster-step aware non-coverage components. Under the autoregressive model, if a partial reasoning prefix \( \mathbf{S}^{1:l} \notin \hat{\mathcal{C}}^{(l)}(X) \), then all subsequent prefixes also fail to be covered, i.e., \( \mathbf{S}^{1:k} \notin \hat{\mathcal{C}}^{(k)}(X) \) for all \( k > l \). This cascading structure allows us to express the full-sequence non-coverage event as a disjoint union of first-step failures, leading to the following decomposition and probability bound.

\begin{proposition}\label{prop_errordecom}
Let \( \hat{\mathcal{C}}^{(l)}(X) \subseteq \mathcal{Y}^l \) be the stepwise conformal prediction set at reasoning step \( l \in [L] \), and let \( \hat{h}_l: \mathcal{Y} \to \mathcal{M}_l \) denote the cluster assignment function that maps each token at  $l$-th step to a unique cluster \( m \in \mathcal{M}_l \).  Then the sequence-level non-coverage event decomposes disjointly as:
\begin{footnotesize}
\begin{eqnarray*}
&&\mathbb{P}_{(X, \mathbf{S}) \sim \mathcal{D}} \left[ \mathbf{S} \notin \mathcal{C}^{(L)}(X;\beta) \right]\\
&=& \sum_{l=1}^L \sum_{m \in \mathcal{M}_l}
\mathbb{P}_{(X, \mathbf{S}) \sim \mathcal{D}} \left[
\underbrace{\hat{h}_l(S^{l}) = m,\ 
\mathbf{S}^{1:l} \notin \mathcal{C}^{(l)}(X;\beta),\
\ \mathbf{S}^{1:k} \in \mathcal{C}^{(k)}(X;\beta),  ~\forall k < l}_{E_{l,m}}
\right].
\end{eqnarray*}
\end{footnotesize}

Suppose for each cluster-step pair \( (l, m) \), the following high-probability bound holds over the calibration set \( Z_{\text{val}} \sim \mathcal{D}^n \):
\[
\mathbb{P}_{Z_{\text{val}} \sim \mathcal{D}^n} \left[
\mathbb{P}_{(X, \mathbf{S}) \sim \mathcal{D}}[E_{l,m}] > \epsilon_{l,m}
\right] \leq \delta_{l,m}.
\]

Then, with probability at least \( 1 - \sum_{l=1}^L \sum_{m \in \mathcal{M}_l} \delta_{l,m} \), we have:
\[
\mathbb{P}_{(X, \mathbf{S}) \sim \mathcal{D}} \left[
\mathbf{S} \notin \mathcal{C}^{(L)}(X;\beta)
\right] \leq \sum_{l=1}^L \sum_{m \in \mathcal{M}_l} \epsilon_{l,m}.
\]
\end{proposition}

\begin{proof}
By definition of the full-sequence conformal predictor, a sequence \( \mathbf{S} \) is included in \( \mathcal{C}^{(L)}(X;\beta) \) if and only if all prefixes \( \mathbf{S}^{1:l} \in \mathcal{C}^{(l)}(X;\beta) \) for \( l = 1, \dots, L \). Hence,
\[
\mathbf{S} \notin \mathcal{C}^{(L)}(X;\beta)
\iff \exists\, l \in [L],\ \text{s.t. } \mathbf{S}^{1:l} \notin \mathcal{C}^{(l)}(X;\beta) \text{ and } \forall k < l,\ \mathbf{S}^{1:k} \in \mathcal{C}^{(k)}(X;\beta).
\]

For each such \( l \) and given the reasoning prefix \( \mathbf{S}^{1:l-1} \), each token $S^l$ belongs to a unique cluster \( m = \hat{h}_{l}(S^{l}) \in \mathcal{M}_l \). Define the event:
\[
E_{l,m} := \left\{
\hat{h}_l(S^{l}) = m,\
\mathbf{S}^{1:l} \notin \mathcal{C}^{(l)}(X;\beta),\
\ \mathbf{S}^{1:k} \in \mathcal{C}^{(k)}(X;\beta),~\forall k < l
\right\}.
\]

Then the collection \( \{ E_{l,m} \}_{l,m} \) is a disjoint partition of the non-coverage event. This follows from the autoregressive structure (errors cascade forward) and the uniqueness of cluster assignment per step (i.e., for any given \( \mathbf{S} \) and \( l \), only one cluster \( m \) satisfies \( \hat{h}_l(S^{l}) = m \)). Thus:
\[
\mathbb{I}[\mathbf{S} \notin \mathcal{C}^{(L)}(X;\beta)] = \sum_{l=1}^L \sum_{m \in \mathcal{M}_l} \mathbb{I}[E_{l,m}].
\]
Taking expectation over \( (X, \mathbf{S}) \sim \mathcal{D} \), we get:
\[
\mathbb{P}[\mathbf{S} \notin \mathcal{C}^{(L)}(X)]
= \sum_{l=1}^L \sum_{m \in \mathcal{M}_l} \mathbb{P}[E_{l,m}].
\]

Suppose we can control the non-coverage rate over  each cluster-step pair \( (l, m) \), i.e., the following high-probability bound holds over the calibration set \( Z_{\text{val}} \sim \mathcal{D}^n \):
\[
\mathbb{P}_{Z_{\text{val}}} \left[ \mathbb{P}[E_{l,m}] > \epsilon_{l,m} \right] \leq \delta_{l,m}.
\]

Applying the union bound over all cluster-step pairs \( (l, m) \), we obtain:
\[
\mathbb{P}_{Z_{\text{val}}} \left[
\sum_{l,m} \mathbb{P}[E_{l,m}] > \sum_{l,m} \epsilon_{l,m}
\right] \leq \sum_{l,m} \delta_{l,m}.
\]

Hence, with probability at least \( 1 - \sum_{l,m} \delta_{l,m} \), we conclude:
\[
\mathbb{P}[\mathbf{S} \notin \mathcal{C}^{(L)}(X;\beta)]
\leq \sum_{l=1}^L \sum_{m \in \mathcal{M}_l} \epsilon_{l,m}.
\]
We complete the proof.
\end{proof}

With the above proposition, we need  to estimate the error probabilities $\delta_{l,m}$ and $\epsilon_{l,m}$ in the inequality
\[
\mathbb{P}_{Z_{\text{val}}} \left[ \mathbb{P}[E_{l,m}] > \epsilon_{l,m} \right] \leq \delta_{l,m},
\]
where $E_{l,m}$ denotes the event that a reasoning step $S^{1:l}$ in cluster $m$ at step $l$ falls outside the conformal set $\mathcal{C}^{(l)}(X;\beta)$.
\subsection{PAC Bound Establishment}\label{sect_pac}
Inspired by \citet{park2019pac}, we begin by reformulating the problem as a binary classification task, enabling the application of standard concentration inequalities to our objective. This transformation is natural because the threshold-based inclusion condition
\[
\mathbb{I}\left[S^{1:l} \in \mathcal{C}^{(l)}(X;\beta)\right] = \mathbb{I}\left[\sigma(X, \mathbf S^{1:l}) \geq Q_l(m; \beta)\right]
\]
can be equivalently rewritten as a binary decision made by a classifier applied to the conformity score $t = \sigma(X, \mathbf S^{1:l})$. Define a cluster-conditional binary classifier associated with each $(l, m)$ pair as:
\[
M_{l,m}(t) := \mathbb{I}\left[ t > Q_{l}(m; \beta) \right],
\]
where $\mathbb{I}[\cdot]$ is the indicator function and $t = \sigma(X, \mathbf{S}^{1:l})$ is the conformal score. Substituting this into the definition of $\mathcal{C}^{(l+1)}(X;\beta)$, we obtain the equivalent form:
\[
\mathcal{C}^{(l+1)}(X;\beta) = \left\{
\mathbf{S}^{1:l} a ~\middle|~
M_{l+1, \hat{h}_{l+1}(a)} \left( \sigma(X, \mathbf{S}^{1:l} a) \right) = 1,\;
\mathbf{S}^{1:l} \in \hat{\mathcal{C}}^{(l)}(X)
\right\}.
\]

That is, the decision to retain $\mathbf{S}^{1:l} a$ in $\mathcal{C}^{(l+1)}(X;\beta)$ depends on whether its score exceeds a learned threshold $Q_{l+1}(m; \beta)$ for the cluster $m = \hat{h}_{l+1}(a)$.

As proven in \citep{park2019pac}, although the PAC analysis directly applies Bernstein's inequality to the indicator function $\mathbb{I}[(x, s^{1:l}) \in E_{l,m}]$, this event is equivalently induced by a threshold-based decision rule $M_{l,m}(x, \mathbf s^{1:l}) := \mathbb{I}[\sigma(x, \mathbf s^{1:l}) > Q_{l}(m;\beta)]$.  Making this classifier explicit enables the application of statistical learning theory, such as empirical Bernstein inequalities and VC bounds, and clarifies that the source of conformal failure lies in classifier miscalibration rather than purely distributional noise. This interpretation also aligns the theoretical PAC analysis with the CoVeR optimization design, which learns these thresholds $\{Q_{l}(m;\beta)\}$ over cluster-step partitions.  Thus, estimating $\mathbb{P}[E_{l,m}]$ reduces to estimating the generalization error of this binary classifier over the score distribution of cluster $m$ at step $l$. 


Given the binary classifier representation \( M_{l,m}(t) := \mathbb{I}[t > Q_{l}(m; \beta)] \), we employ the empirical Bernstein inequality to obtain PAC-style generalization guarantees over the cluster-step noncoverage events. This choice is motivated by its ability to yield variance-sensitive bounds that adapt to the observed calibration error variance within each cluster-step pair. Unlike traditional Hoeffding-type bounds, empirical Bernstein exploits low-variance settings to provide tighter control, which aligns well with the CoVeR framework's goal of adaptive confidence allocation across reasoning trajectories.

\begin{lemma}[Empirical Bernstein Inequality \cite{}]
For i.i.d. random variables \( X_1, \dots, X_N \in [0, 1] \) with empirical mean \( \hat{\mu} \) and variance \( \hat{v} \), with probability at least \( 1 - \delta \), there holds
\[
\mathbb{E}[X] \leq \hat{\mu} + \sqrt{\frac{2 \hat{v} \log(3/\delta)}{N}} + \frac{3 \log(3/\delta)}{N},
\]
where input $X$ follows the same distribution with $X_i, i=1,\cdots,N$.
\end{lemma}
Let \( n_{l,m} \) be the size of validation samples  for cluster-step pair \( (l,m) \), and the empirical error and empirical variance be
\[
\hat{\epsilon}_{l,m} := \frac{1}{n_{l,m}} \sum_{i=1}^{n_{l,m}} \mathbb{I}[M_{l,m}(x_i, \mathbf s_i^{1:l}) = 1],
\quad
\hat{v}_{l,m} := \frac{1}{n_{l,m}} \sum_{i=1}^{n_{l,m}} \left( \mathbb{I}[M_{l,m}(x_i, \mathbf s_i^{1:l}) = 1] - \hat{\epsilon}_{l,m} \right)^2.
\]

Let \( X_i := \mathbb{I}[M_{l,m}(x_i, \mathbf s_i^{1:l}) = 1] \), which are i.i.d. Bernoulli variables bounded in \( [0,1] \). By applying the empirical Bernstein inequality to the binary classifier \( M_{l,m} \), and letting \( \hat{\mu} = \hat{\epsilon}_{l,m} \), \( \hat{v} = \hat{v}_{l,m} \), and \( n = n_{l,m} \), we obtain the desired bound with probability at least \( 1 - \delta_{l,m} \).

Additionally, let \( \hat{p}_{l,m} := \frac{1}{N} \sum_{i=1}^{N} \mathbb{I}[\hat{h}_l(S_i^l) = m] \) be the empirical cluster assignment probability at step \( l \). Then, according to Hoeffding's inequality,  with probability at least \( 1 - \zeta_{l,m} \), we have
\[
\mathbb{P}[\hat{h}_l(S^l) = m] \leq \hat{p}_{l,m} + \sqrt{ \frac{ \log (2 / \zeta_{l,m}) }{2N} }.
\]

By combining the above asymptotic bounds, we obtain the following result for each cluster-step pair.

\begin{proposition} Let \( M_{l,m}(x, \mathbf s^{1:l}) := \mathbb{I}[\sigma(x, \mathbf s^{1:l}) > Q_l(m; \hat{\beta})] \) be the binary classifier, and define the true cluster-step failure probability as
\[
\mathbb{P}[E_{l,m}] := \mathbb{P}[\hat{h}_l(S^l) = m] \cdot \mathbb{P}\left[ M_{l,m}(x, \mathbf s^{1:l}) = 1 \mid \hat{h}_l(S^l) = m \right].
\]

Then, with probability at least \( 1 - \delta_{l,m} - \zeta_{l,m} \), we have:
\[
\mathbb{P}[E_{l,m}] 
\leq \left( \hat{p}_{l,m} + \sqrt{ \frac{ \log(2/\zeta_{l,m}) }{ 2N } } \right)
\cdot 
\left( \hat{\epsilon}_{l,m} + \sqrt{ \frac{2 \hat{v}_{l,m} \log(3/\delta_{l,m})}{n_{l,m}} } + \frac{3 \log(3/\delta_{l,m})}{n_{l,m}} \right).
\]
\end{proposition}

Let \( \hat{\epsilon}_{\text{path}} := \frac{1}{N} \sum_{i=1}^N \mathbb{I}[\mathbf s^i \notin \hat{\mathcal{C}}^{(L)}(x_i)] \) be the empirical full-path noncoverage rate. Based the non-coverage rate decomposition in Proposition \ref{prop_errordecom},  we obtain the final theorem.

\begin{theorem}
For any $0<\delta_m<1$ and $0<\zeta_m<1$, with probability at least \( 1 - \delta - \zeta \), there holds
\begin{eqnarray*}
&&\mathbb{P}_{(X, \mathbf S)\sim\mathcal{D}} \left[ \mathbf S \notin \hat{\mathcal{C}}^{(L)}(X) \right]\\
&\lesssim&
\hat{\epsilon}_{\text{path}} + \sum_{l,m} \hat{p}_{l,m}
\left(
\sqrt{ \frac{2\hat{v}_{l,m} \log(3/\delta_{l,m})}{n_{l,m}} }
+ \frac{3\log(3/\delta_{l,m})}{n_{l,m}}
\right)
+ \sum_{l,m} \hat{\epsilon}_{l,m}\sqrt{ \frac{ \log(2/\zeta_{l,m}) }{ 2n_{l,m} } }.
\end{eqnarray*}
\end{theorem}

\newpage


\end{document}